\documentclass[letterpaper, 10 pt, journal, twoside]{ieeetran} 

\IEEEoverridecommandlockouts

\let\labelindent\relax
\usepackage[inline]{enumitem}

\usepackage[english]{babel}

\usepackage[dvipsnames]{xcolor}

\usepackage{graphicx} % Required for inserting images
\usepackage{ragged2e}
\usepackage{fourier}
\usepackage{amssymb}
\usepackage{enumitem}
\usepackage{amsmath}
\usepackage{float}
\usepackage{algorithm}
\usepackage[noend]{algpseudocode}
\usepackage{todonotes}

\usepackage{hyperref}
\usepackage{cleveref}

\usepackage{subcaption}
\usepackage{float}

\usepackage{bm}

\usepackage{tikzscale}
\usetikzlibrary{arrows.meta}
\usetikzlibrary{angles, fit, calc}

\usepackage{pgfplotstable}
\usepackage{pgfplots}
\pgfplotsset{width=7cm,compat=1.18}\usepgfplotslibrary{statistics}

\usepackage{amsthm}

\makeatletter
\def\th@plain{%
  \thm@notefont{}% same as heading font
  \itshape % body font
}
\makeatother

\newtheorem{theorem}{Theorem}
\newtheorem{lemma}[theorem]{Lemma}
\newtheorem{assumption}[theorem]{Assumption}
\newtheorem{corollary}[theorem]{Corollary}
\newtheorem{definition}{Definition}

\newcommand{\atandue}{\mathrm{atan2}}

\newcommand{\tr}[1]{\overset{\triangle}{#1}}
\renewcommand{\mod}{\,\text{mod}}
\newcommand{\rev}[1]{\textcolor{black}{#1}}

\begin{document}

\title{Fast Shortest Path Polyline Smoothing With $G^1$ Continuity and Bounded Curvature}
\author{
    Patrick Pastorelli$^{a}$, Simone Dagnino$^{a}$, Enrico Saccon$^{a}$, Marco Frego$^{b}$, Luigi Palopoli$^{a}$
% \thanks{Simone Dagnino and Patrick Pastorelli contributed equally to this work.}
\thanks{
    Manuscript received: September 15, 2024; 
    Revised: December 20, 2024; 
    Accepted: January 30, 2025. \\
    This paper was recommended for publication by Editor Lucia Pallottino upon evaluation of the Associate Editor and Reviewers’ comments.
}
\thanks{
    $^{a}$Department of Information Engineering and Computer Science, University of Trento, Trento, Italy.
    \{patrick.pastorelli, simone.dagnino\}@studenti.unitn.it, {enrico.saccon, luigi.palopoli}@unitn.it.
}
\thanks{
    $^{b}$Faculty of Engineering, Free University of Bozen-Bolzano, Bolzano, Italy.
    marco.frego@unibz.it.
}

\thanks{2025 IEEE. Personal use of this material is permitted. Permission from IEEE must be obtained for all other uses, in any current or future media, including reprinting/republishing this material for advertising or promotional purposes, creating new collective works, for resale or redistribution to servers or lists, or reuse of any copyrighted component of this work in other works.
The definitive version~\cite{RAL} is available at: \url{https://www.doi.org/10.1109/LRA.2025.3540531}}
}
\markboth{IEEE Robotics and Automation Letters. Preprint Version. Accepted January, 2025}{Pastorelli \MakeLowercase{\textit{et al.}}: Fast Shortest Path Polyline Smoothing With $G^1$ Continuity and Bounded Curvature }

\maketitle

% This work has been submitted to the IEEE for possible publication. Copyright may be transferred without notice, after which this version may no longer be accessible.
% \vspace{0.5cm}

\begin{abstract}
  \rev{In this work, we propose the Dubins Path Smoothing (DPS) algorithm, a novel and efficient method for smoothing polylines in motion planning tasks. DPS} applies to motion planning
  of vehicles with bounded curvature. In the paper, we show that the generated path:
  \begin{enumerate*}
      \item has minimal length, 
      \item is $G^1$ continuous, and
      \item is collision-free by construction, under mild hypotheses.
  \end{enumerate*}
  We compare our solution with the state-of-the-art and show its convenience
  both in terms of computation time and of length of the compute path. % \rev{MANCA DIRE QUI CHE SI CHIAMA DSP}
\end{abstract}

\section{Introduction}
\label{sec:intro}
Motion planning is a fundamental task for many applications, ranging from robotic arm manipulation~\cite{10132587} to autonomous vehicle navigation~\cite{frego:2020}. The goal is to find a feasible (or optimal) path or trajectory to move an agent from a start to a target position, avoiding obstacles.%along the way.

 When considering mobile robots subject to a minimum turning radius and moving at constant speed, the optimal path connecting two different configurations is given by Dubins curves~\cite{dubins}. When the robot has to travel across an area cluttered with obstacles, a popular strategy is to first identify a collision-free polyline joining an initial and a final configuration and then interpolate it by a smooth path (e.g., by a sequence of Dubins manoeuvres), which respect the kinematic constraint of the vehicle. Finding a good interpolation requires a good compromise between the computation time and quality of the resulting path. Some of the existing approaches, prioritise the quality of the path~\cite{Bevilacqua:2016}, but the high computational cost could make them unfit for real--time path generation~\cite{JOUR}. Conversely, methods that prioritise computational efficiency can sometimes produce paths that lack formal mathematical properties or feasibility guarantees~\cite{8407982}, compromising the quality of motion. Indeed, an important additional problem is that even if the initial polyline is guaranteed to be collision free, this is not necessarily true for the interpolated path~\cite{Bertolazzi:2019b}.

In this paper, we propose the Dubins Path Smoothing (DPS) algorithm, a novel solution designed to achieve a fast interpolation of a polyline with a sequence of motion primitives respecting  $G^1$ continuity and bounded curvature constraint. 
An example output of the algorithm is in Figure~\ref{fig:p1}, where a polyline of points, shown in blue, is approximated by a sequence of Dubins manoeuvres (in red) respecting the non-holonomic constraint of the mobile robot. \rev{DPS is a general algorithm for polyline approximation that can work on polylines given as input, e.g., polylines extracted from sample based or from cell decompositions methods. It does not account for the initial and final headings of the robot, but these can be connected to the remaining of the polyline by using Dubins manoeuvres.}

% \lui{This example is not strictly necessary. You can remove the figure in case you lack space}
\begin{figure}[htp]
    \centering
    \begin{tikzpicture}
    % Define the points
    \coordinate (p0) at (   0,  0);
    \coordinate (p1) at (   1,  3);
    \coordinate (p2) at (   5,  0.5);
    \coordinate (p3) at (7.25,  3.5);

    % \coordinate (pi) at (-0.3, -0.9);
    % \coordinate (pf) at ( 7.7,  4.1);

    \coordinate (q0) at (0.6983776575305, 2.0951329725916);
    \coordinate (q1) at (1.8099692463654, 2.4937692210216);
    \coordinate (q2) at (4.4145572538608, 0.865901716337);
    \coordinate (q3) at (5.4157027037873, 1.0542702717165);

    \coordinate (c1) at (1.4121243611714,1.8572174047113);
    \coordinate (c2) at (4.8142285912329,1.5053758561323);

    % Draw the lines
    \draw[blue, line width=1.75pt] (p0) -- (p1);
    \draw[blue, line width=1.75pt] (p1) -- (p2);
    \draw[blue, line width=1.75pt] (p2) -- (p3);

    % Draw the circles
    \draw (c1) circle [radius=0.75];
    \draw (c2) circle [radius=0.75];

    % Draw the radius
    \draw[] (c1) -- node[above]{$r$} (q0);
    \draw[] (c1) -- node[above]{$r$} (q1);
    \draw[] (c2) -- node[above]{$r$} (q2);
    \draw[] (c2) -- node[above]{$r$} (q3);

    % Draw the pd lilne
    \draw[red, line width=0.75pt] (p0) -- (q0);
    \draw[red, line width=0.75pt] (q1) -- (q2);
    \draw[red, line width=0.75pt] (q3) -- (p3);

    \draw [red, line width=0.75] (q1) arc[start angle=58, end angle=161.5, radius=0.75];
    \draw [red, line width=0.75] (q2) arc[start angle=238, end angle=323.13, radius=0.75];

    % \draw [dashed, line width=0.75] (pi) -- (p0);
    % \draw [dashed, line width=0.75] (p3) -- (pf);

    % Label the points    
    \fill (p0) circle (2pt) node[right] {$p_0$};
    \fill (p1) circle (2pt) node[above right] {$p_1$};
    \fill (p2) circle (2pt) node[below right] {$p_2$};
    \fill (p3) circle (2pt) node[right] {$p_3$};

    \fill (q0) circle (1.5pt) node[left] {$q_{1_1}$};
    \fill (q1) circle (1.5pt) node[above right] {$q_{1_2}$};
    \fill (q2) circle (1.5pt) node[below left] {$q_{2_1}$};
    \fill (q3) circle (1.5pt) node[below right] {$q_{2_2}$};
    
    % \fill (p0) circle (2pt) node[right] {$p_i$};
    % \fill (p1) circle (2pt) node[above right] {$p_{i+1}$};
    % \fill (p2) circle (2pt) node[below right] {$p_{i+2}$};
    % \fill (p3) circle (2pt) node[right] {$p_{i+3}$};

    % \fill (q0) circle (1.5pt) node[left] {$q_{2(i-1)}$};
    % \fill (q1) circle (1.5pt) node[above right] {$q_{2(i-1)+1}$};
    % \fill (q2) circle (1.5pt) node[below left] {$q_{2(i-1)}$};
    % \fill (q3) circle (1.5pt) node[below right] {$q_{2(i-1)+1}$};
    
    \fill (c1) circle (1.5pt) node[below] {$c_1$};
    \fill (c2) circle (1.5pt) node[above] {$c_2$};

\end{tikzpicture}
    \caption{In blue the original 4-points polyline $P=\{p_0,\hdots,p_3\}$ and in red the smoothed path from DPS.}
    \label{fig:p1}
\end{figure}
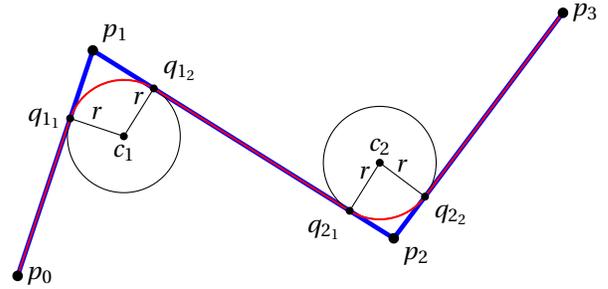

The key findings of the paper are the following: 
\begin{enumerate*}
    \item the length of the interpolated path is upper bounded by that of the polyline;
    \item under assumptions on the polyline, the interpolated path is guaranteed to be collision-free if the polyline is also collision-free, 
    \item the solution is robust: if we cut away some pieces of the solution, the remaining segments are not affected;
    \item given some constraints on the construction of the smoothed path, we prove its optimality, i.e., it is the shortest one;
    \item since the computed path is a combination of segments and arcs of circle, an offset of the solution is a path with the same characteristics and curve primitives. This is a fundamental property for tasks such as CNC where the set of available curve-classes is restricted.
\end{enumerate*}

In the paper, we describe the algorithm and offer a comprehensive theoretical analysis supporting the aforementioned claims. In addition, we provide both a sequential and parallel implementation\footnote{\url{https://github.com/KubriksCodeTN/DPS}} of the algorithm and show its efficiency and effectiveness against existing methods. Finally, we report our experimental results derived from its execution with a real robot with limited hardware to show in a concrete use-case its computational efficiency.

The remainder of this paper is structured as follows. In Section~\ref{sec:related}, we present a review of the state of the art on algorithms for interpolation and approximation of polylines. In Section~\ref{sec:problem}, we detail the tackled problem and in Section~\ref{sec:method} we show the proposed algorithm, also providing formal mathematical support to the claims. Experimental results and performance evaluation are presented in Section~\ref{sec:evaluation}, and in Section~\ref{sec:conclusion} we discuss the potential algorithm applications and future work.

\section{Related Work}
\label{sec:related}

Nonholomic vehicles cannot follow a sequence of segments as in a polyline, but must find a combination of smoothly joint segments to correctly follow the path. The problem of finding the shortest path with $G^1$ continuity to go from an initial position to a final one, with fixed initial and final orientation and also fixed minimum turning radius, was formalized in~\cite{markov1887some} and the solution was formulated in~\cite{dubins}, now known as the Markov-Dubins path.\\
A generalization of the problem, in which the goal is to interpolate multiple points, is called the Multi-Point Markov-Dubins Problem (MPMDP)~\cite{8989830, JOUR}. Some works have focused on optimal control formulations~\cite{JOUR,Goaoc}, which are then computed using Non-Linear Programming solvers~\cite{JOUR}, or Mixed-Integer Non-Linear Programming ones~\cite{8989830}. While this approach allows for finding precise solutions, it also comes with considerable computational costs and it may not converge. A different approach is presented in~\cite{8989830}, in which the authors sample possible angles for each waypoint and, by exploiting dynamic programming and refinements on the solutions, they are able to compute the path with great precision and low computational costs. An extension to their work allows for using GPGPUs in order to parallelize the computation and further improve performance~\cite{9529810}.

A related subproblem that has gained a lot of attention is the 3-Points Dubins Problem (3PDP)~\cite{Goaoc,7799192}, which is a sub-problem of the MPMDP focusing on interpolating 3 points with only the middle angle unknown. It can be used to change parts of the interpolated problem by inserting another point and connecting the previous and next node to the new one, while not changing the rest of the path. 

% In \cite{Goaoc}, it was proved that the optimal path always contains an arc divided into 2 halves by the intermediate point. This result has been heavily exploited in later works (e.g. \cite{PARLANGELI2019295, 7799192}), in order to find novel and efficient approaches to find the optimal heading for the intermediate point. In \cite{7799192} the authors provide a new geometrical analysis of the problem with an algorithm iteratively refined an approximation of the optimal heading. In \cite{CHEN2019368}, the problem solved by finding the zeros of some polynomials with degrees of at most 20. 

Previous work has dealt also with the MPMDP by proposing solutions which embed in the same procedure both the sampling of the points and the following interpolation based on Dubins curve. In \cite{8820677}, the authors propose an approach based on a modified essential visibility graph, in which they connect circles centered in the obstacles' vertexes. Similarly, in \cite{7748231}, the authors construct a so called Dubins path graph to hold all the different pieces of all the possible solutions, and then they run Dijkstra algorithm on a weight matrix to find the shortest path.
In \cite{8087897}, the authors propose a revisited A* algorithm, which takes into account the nonholonomic constraints. This new implementation of A* is applied to a weighted graph which represents the scenario by interpolating a Dubins path between each node and obstacle.
A combined solution for point sampling, using RRT*, and interpolation with Dubins curves is proposed in \cite{9538414}. A post-process algorithm prevents possible collisions from happening during the interpolation step by computing a Dubins path pairwise considering the initial and final angles aligned with the segments that form the path obtained from RRT*. %However, this is a greedy approach that could lead to a sub-solution.
A modified version of RRT* that takes into account nonholonomic constraints is proposed in \cite{8918671}. Here, the authors implemented an RRT* version that creates the Dubins path together with the sampling itself; this is achieved by connecting each new point with a feasible Dubins curve rather than a straight line as in classic RRT*. Similarly, in \cite{10311148,pallottino:2016}, the points are sampled in the 3D space from which the Dubins curves are directly created, representing feasible flight procedures defined by the International Civil Avionics Organization. % One of these flight procedures (track-to-fix leg) represents a use-case for our proposed algorithm, as it doesn't require to go through every point of the polyline.

Focusing on the problem of fitting (or smoothing) a polyline, an approach similar to our was presented in~\cite{1626645}, in which the authors propose an algorithm with complexity $O(n^2)$ to smooth the polyline by first finding maximum convex subsections of the polyline and then approximating them with curves. Another analogous approach was presented in~\cite{8407982}, in which the authors use Fermat Spirals to approximate the polyline with a series of differentiable curves, providing an easy to follow path for UAVs and UMVs. 
While these approaches may provide good approximations minimizing variations in the curvature, based on the theoretical complexity presented, they are not as performing as ours, and they do not discuss correctness and optimality. 
In recent times, it was shown that a curvilinear coordinate transformation could be used to iteratively adapt the polyline to the reference frame~\cite{wursching2024robust}. The authors consider the polyline as the control polygon of a B-spline, which is then iteratively refined in conjunction with a sampling step to reduce the curvature and create a smooth $C^2$ curve. Nevertheless, the authors neither consider length optimality, nor guarantee the curvature bound.

A final comparison can be made with biarcs~\cite{Bertolazzi:2019}, which provide a fast solution to a similar problem of fitting points given an initial and final angles with $G^1$ continuity. However, they do not consider bounded curvature nor minimize the length of the solution. 

\newcommand{\seg}[1]{#1}
\newcommand{\vect}[1]{#1}
\newcommand{\arc}[1]{#1}

\renewcommand{\seg}[1]{\overline{#1}}
\renewcommand{\vect}[1]{\overrightarrow{#1}}
\renewcommand{\arc}[1]{\widearc{#1}}

\section{Problem Statement}
\label{sec:problem}
In this section, we state the problem and show how it can be decomposed into sub-problems thanks to Bellman's principle of optimality, without loss of generality.

\begin{figure}[!b]
    \centering
    \begin{tikzpicture}
    % Define the points
    \coordinate (pi) at ( 0  , 0);
    \coordinate (pm) at ( 7  , 0);
    \coordinate (pf) at ( 2.5 , 3);
    \coordinate (p0) at (-0.5 , 0);
    \coordinate (p1) at ( 2.07, 3.28);

    \coordinate (q0) at (2.8715, 0);
    \coordinate (q1) at (3.5649, 2.29);

    \def\radius{1.25}

    \coordinate (c1) at (2.8715, \radius);

    % Draw the axis 
    \draw[->] (0,0) -- (8,0) node [below] {$x$};
    \draw[->] (0,-0.3) -- (0,3.5) node [left] {$y$};
    
    % Draw the lines
    \draw[] (pi) -- (pm);
    \draw[] (pm) -- (pf);

    % Draw the circles
    \draw (c1) circle [radius=\radius];

    % Draw the pd segments
    \draw[red, line width=1pt] (pi) -- (q0);
    \draw[red, line width=1pt] (q1) -- (pf);

    % The arc 
    \draw [red, line width=1] (q0) arc[start angle=-90, end angle=56.31, radius=\radius];

    % Continuation of lines
    \draw [dashed, line width=0.75] (p0) -- (pi);
    \draw [dashed, line width=0.75] (pf) -- (p1);

    % The radius
    \draw [line width=0.5] (c1) -- node [left] {$r$} (q0);
    \draw [line width=0.5] (c1) -- node [left] {$r$} (q1);

    % The distance between tangent points and middle point
    \draw [line width=0.5] (pm) -- node [below] {$l$} (q0);
    \draw [line width=0.5] (pm) -- node [above] {$l$} (q1);

    % The distance between pm and C 
    \draw [line width=0.5, dashed] (pm) -- node [above] {$d$} (c1);

    % The angles 
    \pic [draw=blue, angle eccentricity=2, line width=1] {angle = pf--pm--pi};
    \node[] at (6.1,0.45) {$\textcolor{blue}{\alpha_m}$};
    \coordinate (tmp) at (8,0);
    \pic [pic text=$\textcolor{green}{\theta_m}$, draw=green, angle eccentricity=1.5, line width=1] {angle = tmp--pm--pf};

    % Label the points    
    \fill (pi) circle (2pt) node[below right] {$p_i$};
    \fill (pm) circle (2pt) node[below] {$p_m$};
    \fill (pf) circle (2pt) node[below] {$p_f$};

    \fill (q0) circle (1.5pt) node[below] {$q_{m_1}$};
    \fill (q1) circle (1.5pt) node[above right] {$q_{m_2}$};
    
    \fill (c1) circle (1.5pt) node[left] {$C$};
    
\end{tikzpicture}
    \caption{Visual representation of the standard placement for the 3-points sub-problem.}
    \label{fig:p3}
\end{figure}
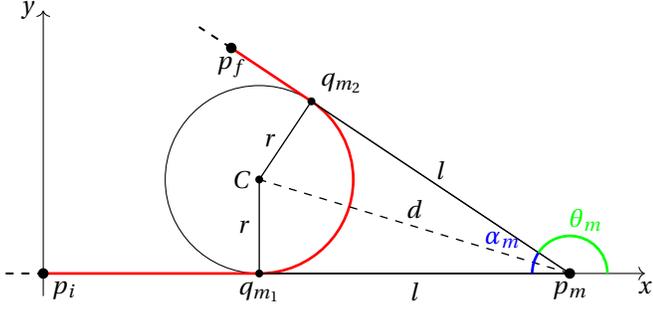

\begin{definition}[Polyline]
A polyline $P = (p_0, p_1, \hdots, p_{n-1})$ is an ordered sequence of $n > 2$ points $p_j=(x_j,y_j) \in \mathbb{R}^2$.
\end{definition}

%%%%%%%%%%%%%%%%%%%%%%%%%%%%%%%%%%%%%%%%%%%%%%%%%%%%%%%%%%%%%%%%%%%%%%%%%%%%%%%%%%%%

\begin{assumption}[Conditions on the polyline]
\label{ass:polylineConditions}
Any triplet of consecutive points $p_{j-1}, p_j, p_{j+1}$ of the polyline $P$ is not-aligned, and, given a scenario with obstacles, $P$ is collision-free w.r.t. the obstacles by design (e.g., given by RRT*~\cite{rrtStar}).
\end{assumption}
We will later provide a formal proof of the minimum size of which the obstacles should be inflated in order to produce a collision-free polyline $P$ as in Assumption~\ref{ass:polylineConditions}.

%%%%%%%%%%%%%%%%%%%%%%%%%%%%%%%%%%%%%%%%%%%%%%%%%%%%%%%%%%%%%%%%%%%%%%%%%%%%%%%%%%%%

\begin{definition}[Problem]
Let $P$ be a collision-free polyline \rev{taken as input to the problem}. The goal is to find a feasible path $\mathcal{P}$ approximating the polyline $P$ with a shortest path for a non-holonomic robot subject to a minimum turning radius $r$, respecting $G^1$ continuity. 
\end{definition}

We define the angles, see Figure~\ref{fig:p3}, 
\[
    \theta_j=\atandue(y_{j+1}-y_j,x_{j+1}-x_j),\quad j<n-1
\]
as the direction between point $p_j$ and point $p_{j+1}$, and 
\[
    \alpha_j = \pi - (\theta_j-\theta_{j-1})
\]
as the internal angle formed between $\vect{p_{j-1}p_j}$ and $\vect{p_jp_{j+1}}$.

To tackle the problem, without loss of generality, we can consider the sub-problem of finding an approximate path $Q$ with $G^1$ continuity and bounded curvature for 3 consecutive non-aligned points $p_i, p_m, p_f$ (respectively initial, middle and final points). When they are aligned, the solution coincides with the polyline itself. Then, we construct the circle of fixed radius $r$ tangent to the two lines defined by the polyline $\{p_i, p_m, p_f\}$, as shown in Figure~\ref{fig:p3}. This intersection defines the tangent points $q_{m_1}$ and $q_{m_2}$ in proximity of point $p_m$. 
Considering sub-problems composed of 3 points allows for an easier and more efficient computation of the approximate path, since each sub-problem can be solved separately and hence computed in parallel.
We want to construct the shortest path with $G^1$ continuity and bounded curvature interpolating with line segments and circular arcs the sequence of ordered points $Q = (p_i, q_{m_1}, q_{m_2}, p_f)$ shown in Figure~\ref{fig:p3}. The designed path is feasible for a non-holonomic robot to follow by construction. Indeed, every arc $\arc{q_{m_1}q_{m_2}}$ results from a circle with the radius equal to the minimum turning radius $r$ of the robot, and the arrival angle, after every arc, corresponds to the slope of the next straight segment in the path. The $G^1$ continuity and the bounded curvature ensure the feasibility of the path. 

It is important to notice that the segments between the points of $P$ must be long enough to contain the tangent points in the correct order. For instance, in the case of Figure~\ref{fig:p3}, the segment $\seg{p_ip_m}$ must be longer than the segment $\seg{q_{m_1}p_m}$. We ensure this in Lemma~\ref{lem:exConstraint}.

%%%%%%%%%%%%%%%%%%%%%%%%%%%%%%%%%%%%%%%%%%%%%%%%%%%%%%%%%%%%%%%%%%%%%%%%%%%%%%%%%%%%

\section{Solution formalization}
\label{sec:method}

In this section, we will formalize the 3-points sub-problem and its solution. 
In order to ease the following demonstration, without loss of generality, we shall translate and rotate the initial configuration for the 3-points sub-problem. We start by shifting the points by the coordinates of $p_i$, bringing it to the origin, then we apply a rotation $R(\varphi)$ to bring the point $p_m$ to the positive $x$-axis. If $p_f$ is below the $x$-axis ($y_f<0$), we reflect it on the $x$-axis and consider it positive. 

%%%%%%%%%%%%%%%%%%%%%%%%%%%%%%%%%%%%%%%%%%%%%%%%%%%%%%%%%%%%%%%%%%%%%%%%%%%%%%%%%%%%

\begin{lemma}[Existence constraint]
\label{lem:exConstraint}

Let $P=(p_i,p_m,p_f)$ be a polyline of 3 non aligned points, and $\alpha_m$ the angle in $p_m$, as in Figure~\ref{fig:p3}. The following constraint must be satisfied to ensure the existence of a $G^1$ approximation of the polyline:
\begin{equation}
    \min\left(\lvert\lvert p_m-p_i\rvert\rvert, \lvert\lvert p_f-p_m\rvert\rvert\right) \geq \frac{r}{\tan(\alpha_m/2)}.
    \label{eq:exConstraint}
\end{equation}
\end{lemma}

\begin{proof}
In the standard setting, the angle $\alpha_m$ is equal to: 
\[
\alpha_m = \pi - \theta_m,\quad \theta_i=0.
\]
The two triangles $\tr{q_{m_1}p_mC}$ and $\tr{q_{m_2}p_mC}$ are congruent since they have three congruent edges, namely $r$, $d$, which they share, and $l$, which is $l = \sqrt{d^2-r^2}$. 

Since,
\begin{equation}
    r=d\sin{\frac{\alpha_m}{2}},\qquad l=d\cos{\frac{\alpha_m}{2}},
\label{eq:radius}
\end{equation}
it follows that:
\begin{equation}
    l=\frac{r}{\tan(\alpha_m/2)},\qquad d=\frac{r}{\sin(\alpha_m/2)}.
\label{eq:lengths}
\end{equation}
%%%
Consequently, the tangent points are: 
% \[
%     q_{m_1}=(x_m-l, 0),\quad q_{m_2}=(x_m-l\cos{\alpha_m}, l\sin{\alpha_m}),
% \]
\[
    q_{m_1}=(x_m-l, 0),\qquad q_{m_2}=(x_m+l\cos{\theta_m},~ l\sin{\theta_m}),
\]
and, it must hold that:
\[
%\begin{cases}
    l \leq x_m , \qquad
    l \leq||p_f-p_i||,
%\end{cases},
\]
i.e., the tangent points are within the polyline segments.

Finally, $\alpha_m\in\{0,\pi\}$ makes the denominator of Equation~\ref{eq:exConstraint} vanish, but this situation corresponds to aligned points, which is excluded by Assumption~\ref{ass:polylineConditions}.
\end{proof}

%%%%%%%%%%%%%%%%%%%%%%%%%%%%%%%%%%%%%%%%%%%%%%%%%%%%%%%%%%%%%%%%%%%%%%%%%%%%%%%%%%%%

% \begin{corollary}[Local existence condition]
% Following Assumption~\ref{ass:dubinsDistance}, if $Q$ approximates $P$ with $G^1$ continuity and bounded curvature, then
% \begin{equation}
%     \tan{\frac{\alpha^*}{2}}\geq \frac{1}{4},
% \end{equation}
% guarantees the existence of the path.
% \end{corollary}

% This bound is conservative, that is, if the points are further apart than $4r$, then the solution exists also for smaller values of $\alpha_j$. On the other hand, for points closer than $4r$, it is possible to show counterexamples.

%%%%%%%%%%%%%%%%%%%%%%%%%%%%%%%%%%%%%%%%%%%%%%%%%%%%%%%%%%%%%%%%%%%%%%%%%%%%%%%%%%%%

In the general case, when considering the whole polyline $P$, we also must ensure that the tangent points in the segment $\seg{p_jp_k}$ exist and are consecutive, that is, $||p_j-q_{j_2}||\leq ||p_j-q_{k_1}||$ and $||p_k-q_{k_1}||\leq ||p_k-q_{j_2}||$.

\begin{corollary}[Global existence condition]
\label{cor:exConstraints}
From Lemma~\ref{lem:exConstraint}, a solution exist if, for every two consecutive points $p_j,\,p_k \in P$, the following constraint is also satisfied: 
\begin{equation}
    ||p_j-p_k||\geq \left\lvert\frac{r}{\tan(\alpha_j/2)}\right\rvert + \left\lvert\frac{r}{\tan(\alpha_k/2)}\right\rvert.
\label{eq:exConstraints}
\end{equation}
\end{corollary}

%%%%%%%%%%%%%%%%%%%%%%%%%%%%%%%%%%%%%%%%%%%%%%%%%%%%%%%%%%%%%%%%%%%%%%%%%%%%%%%%%%%%

\begin{corollary}[Tangent points]
\label{cor:exTanPoints}
If the constraints in Lemma~\ref{lem:exConstraint} and Corollary~\ref{cor:exConstraints} are satisfied, then the tangent points $q_{j_1}, q_{j_2}$ between the circle and the vectors $v_{j_1} = \vect{p_{j-1}p_j}$ and $v_{j_2} = \vect{p_jp_{j+1}}$ in the general position exist and are given by:
\begin{equation}
    q_{j_1} = p_j - l_j\hat{v}_{j_1}\qquad q_{j_2} = p_j + l_j\hat{v}_{j_2}
\label{eq:exTanPoints}
\end{equation}
\end{corollary}
where $\hat{v}_{j_1}$ is the unit direction vector from $p_{j-1}$ to $p_j$, and $\hat{v}_{j_2}$ the unit direction vector from $p_j$ to $p_{j+1}$. 

%%%%%%%%%%%%%%%%%%%%%%%%%%%%%%%%%%%%%%%%%%%%%%%%%%%%%%%%%%%%%%%%%%%%%%%%%%%%%%%%%%%%
\vspace{0.25cm}
Our DPS algorithm generates a path that does not interpolate $P$, but it is possible to derive an upper bound of its distance from $P$.

\begin{corollary}[Distance from the polyline points]
\label{cor:WdistanceFromPoints}
The path $Q$ has distance $D_j$ from the intermediate point $p_j$ of the polyline $P$ given by:
\begin{equation}
    D_j = d-r = r\left(\frac{1}{\sin(\alpha_m/2)}-1\right) = r\cdot\text{excsc}(\alpha_m/2)
\end{equation}
\end{corollary}

% \begin{proof}
% Due to triangle properties it can be shown that the distance from the center of the circle on which the arc found by our algorithm lies to the intermediate point $p_i$ of the polyline $P$ is $r * \csc(\theta_i)$. This can also be seen in \cref{fig:p3}.

% From this it follows that the distance from said circle to the point $p_i$ is:
% \begin{equation}
%     r * (\csc(\theta_i) - 1) = r * \text{excsc}(\theta_i)
% \end{equation}
% \end{proof}

%%%%%%%%%%%%%%%%%%%%%%%%%%%%%%%%%%%%%%%%%%%%%%%%%%%%%%%%%%%%%%%%%%%%%%%%%%%%%%%%%%%%

The lemma and its corollaries are implemented in Algorithm~\ref{algo:ithpiece}, which runs in constant time.

\begin{algorithm}[ht]
\begin{algorithmic}[1]
\Function{solve\_3\_points}{$p_i=(x_i, y_i)$, $p_m=(x_m, y_m)$, $p_f=(x_f,y_f)$}
\State $v_1 \gets [x_m - x_i, y_m - y_i]^T$
\State $v_2 \gets [x_f - x_m, y_f - y_m]^T$
\State $\hat{v_1} \gets v_1/|| v_1||$
\State $\hat{v_2} \gets v_2/|| v_2||$\vspace{0.5em}
\State $l \gets r \dfrac{|| v_1 \times v_2||}{(v_1 \cdot v_2) +||v_1||\,|| v_2||}$\vspace{0.5em}
\State $q_{m_1} \gets p_m - l \hat{v_1}$
\State $q_{m_2} \gets p_m + l \hat{v_2}$
\State $Q \gets \left\{\overline{p_iq_{m_1}} + \widearc{q_{m_1}q_{m_2}}\right\}$
\\ \hspace{3.5mm} \Return{$Q$}
\EndFunction
\end{algorithmic}
\caption{The pseudo-code to compute the $i^{\text{th}}$ piece of the path $\mathcal{P}$ (i.e., $Q_i$), in the general case.}
\label{algo:ithpiece}
\end{algorithm}
\textbf{On the complexity of the algorithm}. 
Since the computation of each individual segment of the path can be performed in constant time, the overall complexity of the algorithm to compute all $n-1$ segments of the polyline $P$ is linear with respect to the number of points $n$. Since every time a new problem is provided, each segment of the path must be computed, the algorithm's complexity is bounded by $n$ in both the best and worst cases. Therefore, the algorithm operates in \(\Theta(n)\) time.
%%%%%%%%%%%%%%%%%%%%%%%%%%%%%%%%%%%%%%%%%%%%%%%%%%%%%%%%%%%%%%%%%%%%%%%%%%%%%%%%%%%%

\begin{figure}[htp]
    \centering
    \begin{tikzpicture}[line width = 0.75]

\coordinate (p0) at (0,0);
\coordinate (p1) at (7.1384312561088015, -2.300608441653842);
\coordinate (p2) at (6.686028448495419, 1.8376679748828215);
\coordinate (q0) at (4.935687329310115, -1.5906973854314983);
\coordinate (q1) at (6.886925166746531, 1.308301491320322e-17);
\coordinate (c) at (5.395809017640884, -0.163011134209738);
\coordinate (H) at (6.238905397404775, -2.0107048606325297);
\coordinate (B) at (4, 1.5491657977445739);
\coordinate (ch1) at (3.958541906171354, -2.3872905223375125);
\coordinate (ch2) at (7.8062170595707405, 0.7494295069047394);
\def\radius{1.5}

\def\widthConstr{0.2}
\def\widthAxis{0.2}
\def\widthPath{1.25}

\def\colorL{blue}
\def\colorD{green}
\def\colorAlpha{black}
\def\colorPath{red}

% Draw the axis 
\draw[->, line width = \widthAxis] (0,0) -- (8,0) node [below] {$x$};
\draw[->, line width = \widthAxis] (0,-2.5) -- (0,2) node [left] {$y$};

% Draw the lines of the problem
\draw[black] (p0) -- (q0);
\draw[\colorL] (q0) -- (p1);
\draw[\colorL] (p1) -- (q1);
\draw[black] (q1) -- (p2);
\draw[black] (ch1) -- (ch2);

% Draw the bisector
\draw[\colorD, line width = \widthConstr] (p1) -- (c);
\draw[black, dashed, line width = \widthConstr] (c) -- (B);

% Add the labels of the problem
\node at (p0) [left] {$p_i$};
\node at (p1) [below] {$p_m$};
\node at (q1) [below right] {$q_{m_2}$};
\node at (p2) [right] {$p_f$};
\node at (q0) [left] {$q_{m_1}$};
\node at (c) [left] {$c$};

% Draw the height
\draw[orange] (H) -- (q1);
\filldraw[orange] (H) circle (1pt);
\node at (H) [below, orange] {$H$};

% Circle
\draw[black] (c) circle (\radius);
\filldraw[black] (c)  circle (1pt);
\draw[black, line width = \widthConstr] (c) -- node[left]{$r$} (q0);
\draw[black, line width = \widthConstr] (c) -- node[above]{$r$} (q1);

% Draw path
\draw[\colorPath, line width = \widthPath] (p0) -- (q0);
\draw[\colorPath, line width = \widthPath] (q0) arc[start angle=-107.863346216765155, end angle=6.238888284396, radius=\radius];

% Draw the angles
\pic [pic text = $\textcolor{\colorAlpha}{\alpha}$, draw=\colorAlpha, angle eccentricity=1.5, line width=1] {angle = p2--p1--p0};
\pic [pic text = $\textcolor{black}{\theta_i}$, draw=black, angle eccentricity=2.2, line width=1] {angle = p1--p0--q1};
\coordinate (tmp1) at (8, -2.300608441653842);
\pic [pic text = $\textcolor{black}{\theta_m}$, draw=black, angle eccentricity=1.5, line width=1] {angle = tmp1--p1--q1};
\draw[black, dashed, line width=0.05] (p1) -- (tmp1);

% Draw the points of the problem
\filldraw[black] (p0) circle (1pt);
\filldraw[black] (p1) circle (1pt);
\filldraw[black] (p2) circle (1pt);
\filldraw[black] (q0) circle (1pt);
\filldraw[black] (q1) circle (1pt);

% Add additional labels for lengths
\node[\colorL] at (5.5, -2) {$l$};
\node[\colorL] at (7.2, -1) {$l$};
\node[\colorD] at (6.3, -0.9) {$d$};
\node[\colorPath] at (2.5, -1.1) {$\ell_2$};
\node[\colorPath] at (5.7, -1.3) {$\ell_3$};

\end{tikzpicture}
    \caption{In red, an example of the J Dubins path type in the standard setting used in Theorem~\ref{th:DubinsPath}.}
    \label{fig:dubins}
\end{figure}
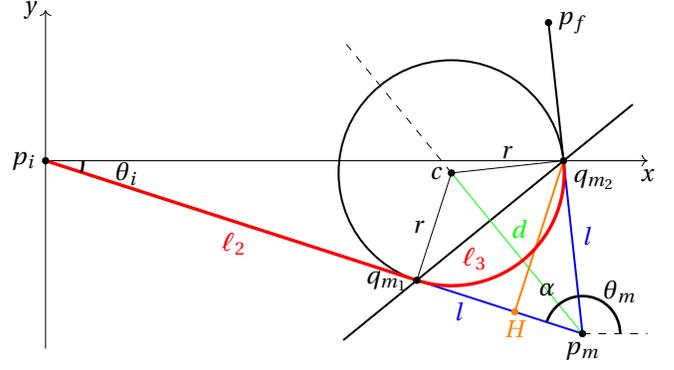

\begin{assumption}[Minimum distance between points]
\label{ass:dubinsDistance}
Let $r$ be the minimum turning radius of the robot, $p_j$ and $p_k$ two consecutive points of the polyline $P$ and $q_{k_2}$ the first tangent point after $p_k$. We assume that the points $p_j$ and $q_{k_2}$ have a distance of at least $4r$:
\[
||p_j-q_{k_2}||\geq 4r.
\]
% Given the results in~\cite{DubinsSet,PARLANGELI2024105814,Piazza:2024}, there exists only Dubins path connecting $p_i-q_{m_2}$ with angles $\theta_i$ and $\theta_m$ in the form CSC, if the distance between points is at least $4r$
\end{assumption}
%%%
\rev{This (sufficient) condition is a desirable property and it is called \emph{long distance points hypothesis} or \emph{far case} in the literature~\cite{PARLANGELI2024105814,PARLANGELI2019295,8989830,9529810,DubinsSet}. It is useful to avoid unnatural loops with CCC segments, and close points can be easily merged into a single representative via point. When Assumption~\ref{ass:dubinsDistance} is violated, DPS may still yield the shortest path, although no guarantees can be provided, unless Corollary~\ref{cor:exConstraints} is satisfied.}

%%%%%%%%%%%%%%%%%%%%%%%%%%%%%%%%%%%%%%%%%%%%%%%%%%%%%%%%%%%%%%%%%%%%%%%%%%%%%%%%%%%%

% \begin{figure*}[htp]
%     \centering
%     \input{images/miterOffset}
%     \caption{Two possible scenarios showing in red an obstacle to which a mitered-offset of $h$ has been applied producing the polyline in black. In blue, we see the inscribed circle used by DPS to compute the path colored in green. In orange, the radius of the robot, which is drawn on the bottom left of each picture. It is possible to notice that, not only must the offset be enough for the robot to fit through, i.e., at least $h$, but it must also ensure that the turning radius of the robot allows it to pass without problems.}
%     \label{fig:miterOffset}
% \end{figure*}

\begin{theorem}[Optimality of $\mathbf{Q}$]
\label{th:DubinsPath}
The path $Q$, \rev{satisfying Assumption~\ref{ass:dubinsDistance}}, is a Dubins path, i.e., the shortest one of bounded curvature and $G^1$ continuity. 
\end{theorem}

\begin{proof}
Using the principle of Bellman, it is possible to consider the sub-problem of showing the optimality of the sub-path given by three points $p_i, q_{m_1}, q_{m_2}$, see Figure~\ref{fig:dubins}. If the theorem holds for the sub-problem, then it holds also for the whole path $Q$. In order to do so, we show that the path $p_i, q_{m_1}, q_{m_2}$ is a Dubins path, i.e., the shortest one connecting with $G^1$ continuity and bounded curvature the points $p_i$ with angle $\theta_i$ and $q_{m_2}$ with angle $\theta_m$. The $G^1$ continuity and the bounded curvature are given by construction, and now we prove the length optimality.

To take advantage of the Dubins formulation, we need to recast the problem to Dubins standard setting, i.e., bringing the initial and final points ($p_i$ and $q_{m_2}$, respectively) on the $x$-axis with $p_i$ in the origin. Without loss of generality, we first perform a translation of the points, so that $p_i$ is on the origin. Then, we rotate the points, so that $q_{m_2}$ is on the positive $x$-axis, and finally, we scale the points, so that the radius is unitary, as in~\cite{Piazza:2024}.

We can now show that our path is a path of type J, that is, a path composed of a straight line and a circle. 

Dubins paths can be in the form of CSC, i.e., formed by a circle arc followed by a straight segment and finally by another circle arc, or in the form CCC, i.e., with three circle arcs~\cite{DubinsSet}. Assumption~\ref{ass:dubinsDistance} rules out the cases CCC, hence we have to consider CSC cases only. Starting from a CSC Dubins path, we show the condition for which the first arc has length $\ell_1 = 0$. We use formulas (13) and (19) of ~\cite{Piazza:2024}, herein adapted to our notation:
%%%
\begin{equation*}
   0= \ell_1 =|\vartheta_S-\theta_i|\mod 2\pi ,
\end{equation*}
%%%
which happens when the angle of the line segment $\vartheta_S$ is equal to the initial angle $\theta_i$, in other words, when  $\vartheta_S=\theta_i$. This  implies that the final arc has length $\ell_3$, given by:
\begin{equation*}
   \ell_3 = |\theta_m-\vartheta_S|\mod 2\pi = |\theta_m-\theta_i|= |\pi-\alpha|,
\end{equation*}
%%%
which is in agreement with the presented construction. It remains to show that also the length of the straight line segment is the same in Dubins' formula as well as in our construction. The length of the line segment $\ell_2$ from~\cite{Piazza:2024} is given,  substituting $\vartheta_S=\theta_i$, by:
%%%
\begin{equation*}
    \ell_2= \left(||p_i-p_m||-\sin\theta_m\right)\cos\theta_i+\sin\theta_i\cos\theta_m,
\end{equation*}
%%%
which is simplified with trigonometric manipulations to:
\begin{equation*}
    \ell_2 = ||p_i-p_m||\cos\theta_i-\sin(\theta_m-\theta_i).
\end{equation*}
We have to prove that $\ell_2=||p_i-q_{m_1}||$, that is, that the optimal Dubins length is the same as the length of our construction, so that the two perfectly match. We first note that $||p_i-p_m||\cos(\theta_i)$ corresponds to the segment $||p_i-H||$ of Figure~\ref{fig:dubins}, hence it remains to show that $||q_{m_1}-H||=\sin(\theta_m-\theta_i)$, the latter being equal to $\sin(\pi-\alpha)=\sin\alpha$. The length of the height $h=||q_{m_2}-H||$ is on one hand $h=\ell\sin\alpha$, on the other hand we have $h=||q_{m_1}-q_{m_2}||\cos(\alpha/2)$. Combining the latter two relations, we get
%%%
\begin{equation*}
\begin{array}{rcl}
    ||q_{m_1}-H||&=& ||q_{m_1}-q_{m_2}||\sin(\alpha/2)\\[0.7em]
    % &=&\displaystyle l\sin(\alpha)\frac{\sin(\alpha/2)}{\cos(\alpha/2)}\\[0.8em]
    &=&\displaystyle \frac{r}{\tan(\alpha/2)}\sin\alpha \tan(\alpha/2) = r\sin\alpha,
\end{array}
\end{equation*}
%%%
and since, by hypothesis and without loss of generality, $r=1$, we proved that $||q_{m_1}-H||=\sin\alpha$ and thus that our construction is a Dubins path.
\end{proof}

\begin{corollary}[Optimality of $\mathcal{P}$]
Provided Corollary~\ref{cor:WdistanceFromPoints} and Theorem~\ref{th:DubinsPath}, and assuming the points $p_j\in P$ satisfy Corollary~\ref{cor:exConstraints}, the concatenation of the paths $Q_i$ forming $\mathcal{P}$ is the shortest path with distance $D_j$ or less from point $p_j$ and with bounded turning radius and $G^1$ continuity.
\end{corollary}

%%%%%%%%%%%%%%%%%%%%%%%%%%%%%%%%%%%%%%%%%%%%%%%%%%%%%%%%%%%%%%%%%%%%%%%%%%%%%%%%%%%%

Now, we show the minimum offset required to enlarge the obstacles when using motion planning algorithms, such as visibility graphs~\cite{VisGraphs} and RRT*~\cite{rrtStar}, to create a road-map to ensure a collision-free path with DPS. \rev{Indeed, it is possible to notice from Figure~\ref{fig:reviewer2}, that in difficult scenarios, where other state-of-the-art interpolation algorithms may fail, if the theoretical hypotheses discussed are satisfied, DPS can provide a solution. We shall consider the visibility graph algorithm, in which the road-map is constructed between the vertexes of the polygons representing the (inflated) obstacles. The resulting polyline, is by construction the closest polyline to the obstacles and hence any other polyline not crossing the (inflated) obstacles, will be collision free.}

\begin{figure}[t]
    \centering
    \begin{tikzpicture}

% Set the size of the plot
\begin{axis}[
    axis equal image,
    width=\linewidth,
    axis lines=middle,
    xmin=-0.8, xmax=5.3,
    ymin=-1, ymax=2.85,
    xlabel={$x$},
    ylabel={$y$},
    axis line style={-Stealth},
    ticks=none
]

% Bounding boxes
\fill[red!50] (axis cs:-0.5,-0.1) rectangle (axis cs:5,-1);
\fill[red!50] (axis cs:-0.5,0.1) rectangle (axis cs:1.9,2.5);
\fill[red!50] (axis cs:2.1,0.1) rectangle (axis cs:5,2.5);

% % Obstacles
\fill[lightgray] (axis cs:0,-0.6) rectangle (axis cs:4.5,-1);
\fill[lightgray] (axis cs:0,0.6) rectangle (axis cs:1.4,2);
\fill[lightgray] (axis cs:2.6,0.6) rectangle (axis cs:4.5,2);

% \fill (7,-2) circle (2pt) node[right] {$p_0$};

% % Robot
\fill[blue] (axis cs:1.4520841787630676, -0.012304499431725513) -- 
             (axis cs:1.2548335005682745, -0.6913421787630675) -- 
             (axis cs:0.5757958212369325, -0.49409150056827444) -- 
             (axis cs:0.7730464994317257, 0.18494617876306751) -- cycle;

\fill[green] (axis cs:1.70711,0.792893) -- 
             (axis cs:2.20711,0.292893) -- 
             (axis cs:1.70711,-0.207107) -- 
             (axis cs:1.20711,0.292893) -- cycle;

\coordinate (p0) at ( 0, 0);
\coordinate (p1) at ( 2, 0);
\coordinate (p2) at ( 2, 2);
\coordinate (p3) at ( 2, 2.5);

% MPDP_OLD
% \coordinate (d1) at (4.16523,-0.338572);
% \coordinate (d2) at (5.22377,0.872582);
% \coordinate (d3) at (5.02328,1.78549);
% \draw[blue, line width=0.75pt] (p0) -- (d1);
% \draw [blue, line width=0.75] (d1) arc[start angle=-90, end angle=-40, radius=1.15];
% \draw [blue, line width=0.75] (p1) arc[start angle=-40, end angle=15, radius=1.04];
% \draw[blue, line width=0.75pt] (d2) -- (d3);
% \draw [blue, line width=0.75] (d3) arc[start angle=190, end angle=180, radius=1.04];
% \draw[blue, line width=0.75pt] (p2) -- (p3);

% MPDP
\coordinate (d1) at (0.278954,-0.0396955);
\coordinate (d2) at (1.01394,-0.253198);
\coordinate (d3) at (2.2532,0.986061);
\coordinate (d4) at (2.0397,1.72105);

% \fill (d1) circle (2pt) node[above right] {$d_1$};
% \fill (d2) circle (2pt) node[above right] {$d_2$};
% \fill (d3) circle (2pt) node[above right] {$d_3$};
% \fill (d4) circle (2pt) node[above right] {$d_4$};

\def\pathWidth{1.25pt}

\draw [blue, line width=\pathWidth] (p0) arc[start angle=90, end angle=72, radius=1];
\draw[blue, line width=\pathWidth] (d1) -- (d2);
\draw [blue, line width=\pathWidth] (d2) arc[start angle=-106, end angle=18, radius=1];
\draw[blue, line width=\pathWidth] (d3) -- (d4);
\draw [blue, line width=\pathWidth] (d4) arc[start angle=197, end angle=180, radius=1];
% \draw[blue, line width=\pathWidth] (p2) -- (p3);

% DPS
\coordinate (q0) at (1,0);
\coordinate (q1) at (2,1);
\draw[green, line width=\pathWidth] (p0) -- (1.05, 0);
\draw [green, line width=\pathWidth] (q0) arc[start angle=-90, end angle=0, radius=1];
\draw[green, line width=\pathWidth] (q1) -- (p2);

% Common paths
\draw[orange, line width=\pathWidth, ->] (-0.5,0) -- (p0);
\draw[orange, line width=\pathWidth, ->] (p2) -- (p3);

% Dots and labels
\addplot[only marks, mark=*, mark options={fill=black}] coordinates {(0,0)};
% \node at (axis cs:-0.3,-0.3) {$(0,0)$};

\addplot[only marks, mark=*, mark options={fill=black}] coordinates {(2,0)};
% \node at (axis cs:2.3,-0.3) {$(2,0)$};

\addplot[only marks, mark=*, mark options={fill=black}] coordinates {(2,2)};
% \node at (axis cs:2,2.3) {$(2,2)$};

\end{axis}
\end{tikzpicture}
    \caption{\rev{An example showing three obstacles (in gray), their bounding boxes (in red) computed with~\eqref{eq:minDistance} and two robots in green and blue. The green one uses DPS and is guaranteed not to hit any obstacle. The blue one uses a Dubins interpolation algorithm~\cite{8989830} and may still collide.}}
    \label{fig:reviewer2}
\end{figure}

\begin{assumption}[Convex obstacles]
The obstacles present in the map are convex.
If the obstacles are not convex, a convex hull can be employed to enable Theorem~\ref{th:minOffset}.
\end{assumption}

\begin{figure}[htp]
    \centering
    \def\widthConstr{0.2}
\def\widthObsSpace{0.5}

\def\colorObs{red}
\def\widthObs{1}

\def\colorPolyline{black}
\def\widthPolyline{0.50}

\def\colorPath{green}
\def\widthPath{1}

\def\colorRobCircle{BurntOrange}
\def\widthRobCircle{1}

\def\colorOffCircle{Purple}
\def\widthOffCircle{1}

\def\colorTurning{Periwinkle}

\begin{tikzpicture}[line width = 0.75]

% \coordinate (projQ0) at (4.506767898379358, 1.203257409548815);
% \coordinate (W) at (6.066851522944824, 0.7485452128172216);

\def\robotRadius{0.7}   % h
\def\offsetRadius{1.2}
\def\turningRadius{1.7} % r
\coordinate (obs_p0) at (0, 0);
\coordinate (obs_p1) at (4, 0);
\coordinate (obs_p2) at (3.0, 2.5);
\coordinate (p0) at (0.0, -1.2);
\coordinate (p1) at (5.772439553712282, -1.2);
\coordinate (p2) at (4.292439553712281, 2.5);
\coordinate (q0) at (3.2614835192865503, -1.2);
\coordinate (q1) at (4.839893893791491, 1.1313641498019764);
\coordinate (c) at (3.2614835192865503, 0.5);
\coordinate (obs_q0) at (4.0, -1.2);
\coordinate (obs_q1) at (5.114172029062312, 0.44566881162492444);
\coordinate (B) at (0.8040361708966435, 2.163772855831102);
\coordinate (W) at (4.669197811084313, -0.4530689758188122);
\def\finalThArc{21.80140948635181}

% Draw obstacle
\draw[\colorObs, line width = \widthObs] (obs_p0) -- (obs_p1);
\draw[\colorObs, line width = \widthObs] (obs_p1) -- (obs_p2);
\draw[dotted, line width = \widthObsSpace, \colorObs] (3.5, 0) -- (2.5, 2.5);
\draw[dotted, line width = \widthObsSpace, \colorObs] (3.0, 0) -- (2.0, 2.5);
\draw[dotted, line width = \widthObsSpace, \colorObs] (2.5, 0) -- (1.5, 2.5);
\draw[dotted, line width = \widthObsSpace, \colorObs] (2.0, 0) -- (1.0, 2.5);
\draw[dotted, line width = \widthObsSpace, \colorObs] (1.5, 0) -- (0.5, 2.5);
\draw[dotted, line width = \widthObsSpace, \colorObs] (1.0, 0) -- (0.0, 2.5);
\draw[dotted, line width = \widthObsSpace, \colorObs] (0.5, 0) -- (0.0, 1.25);

% Turning Circle
\draw[\colorTurning, dashed, line width = \widthRobCircle] (c) circle (\turningRadius);
\filldraw[\colorTurning] (c)  circle (1pt);
\draw[\colorTurning, line width = \widthConstr] (c) -- node[above]{$r$} (q1);

% Robot Circle
\draw[\colorRobCircle, dashed, line width = \widthRobCircle] (obs_p1) circle (\robotRadius);
\filldraw[\colorRobCircle] (obs_p1)  circle (1pt);
% \draw[\colorRobCircle, line width = \widthConstr] (obs_p1) -- node[left]{$h$} (obs_q0);

% Offset Circle
\draw[\colorOffCircle, dashed, line width = \widthRobCircle] (obs_p1) circle (\offsetRadius);
\filldraw[\colorOffCircle] (obs_p1)  circle (1pt);
\draw[\colorOffCircle, line width = \widthConstr] (obs_p1) -- node[above left]{$O$} (obs_q0);

% Draw bisector
\draw[dashed, line width = \widthConstr] (p1) -- (B);

% Draw polyline
\draw[\colorPolyline] (p0) -- (q0);
\draw[\colorPolyline, line width = \widthConstr] (q0) -- (p1);
\draw[\colorPolyline, line width = \widthConstr] (p1) -- (q1);
\draw[\colorPolyline] (q1) -- (p2);

% Draw path
\draw[\colorPath, line width = \widthPath] (p0) -- (q0);
\draw[\colorPath, line width = \widthPath] (q1) -- (p2);
\draw[\colorPath, line width = \widthPath] (q0) arc[start angle=-90, end angle=\finalThArc, radius=\turningRadius];

% Draw nodes
\filldraw[black] (p0) circle (1pt);
\filldraw[black] (p1) circle (1pt);
\filldraw[black] (p2) circle (1pt);
\filldraw[black] (q0) circle (1pt);
\filldraw[black] (q1) circle (1pt);
\filldraw[\colorTurning] (c) circle (1pt);
\filldraw[black] (obs_p1) circle (1pt);
\filldraw[black] (W) circle (1pt);
% \filldraw[black] (projQ0) circle (1pt);
% \filldraw[black] (W) circle (1pt);

% Add names
\node[\colorPolyline] at (p0)[below] {$p_i$};
\node[\colorPolyline] at (p1)[below right] {$p_m$};
\node[\colorPolyline] at (p2)[right] {$p_f$};
\node[\colorPolyline] at (c)[left, \colorTurning] {$c$};
\node[\colorPolyline] at (q0)[below] {$q_{m_1}$};
\node[\colorPolyline] at (q1)[right] {$q_{m_2}$};
% \node[\colorPolyline] at (projQ0)[below left] {$z$};
\node[\colorObs] at (obs_p1)[right] {$v$};
\node[\colorPolyline] at (W)[right] {$w$};

% Draw robot
\def\robotX{1.2}
\def\robotY{-1.2}
\def\robotW{\robotRadius*1.38} %sqrt(2) is now 1.38
\def\robotH{\robotRadius*1.38} 
\def\wheelW{0.2}
\def\wheelH{0.1}
\def\wheelX {\robotX - .5*\wheelW}
\def\LwheelY{\robotY + .5*\robotH - 0.25*\wheelH}
\def\RwheelY{\robotY - .5*\robotH - 0.75*\wheelH}
\draw[draw=black, fill=white] (\robotX-.5*\robotW, \robotY-.5*\robotH) rectangle ++(\robotW,\robotH);
\draw[draw=black, fill=white, rounded corners=0.05cm, fill=black] (\wheelX,\LwheelY) rectangle ++(\wheelW,\wheelH);
\draw[draw=black, fill=white, rounded corners=0.05cm, fill=black] (\wheelX,\RwheelY) rectangle ++(\wheelW,\wheelH);
\draw[\colorRobCircle, line width = 0.5] (\robotX, \robotY) circle (\robotRadius);
\draw[\colorRobCircle, line width = \widthConstr] (\robotX, \robotY) -- node[below, \colorRobCircle]{$h$} (\robotX+\robotRadius*0.85, \robotY+\robotRadius*0.5); %sqrt(3)/2 is now 0.85

% Vertex angle
\pic [draw=black, angle eccentricity=1, line width=1] {angle = obs_p2--obs_p1--obs_p0};
% \node[] at (4.5, 0.65)[] {$\alpha_m$};
\pic [draw=black, angle eccentricity=1, line width=1] {angle = p2--p1--p0};
\node[] at (6, -0.65)[] {$\alpha_m$};
\end{tikzpicture}
    \caption{A scenario in which the mitered offset $O$, in purple, is larger than the radius of the robot $h$ in orange. In blue, the turning radius of the robot and in green the path.}
    \label{fig:mitOffsetScenario}
\end{figure}
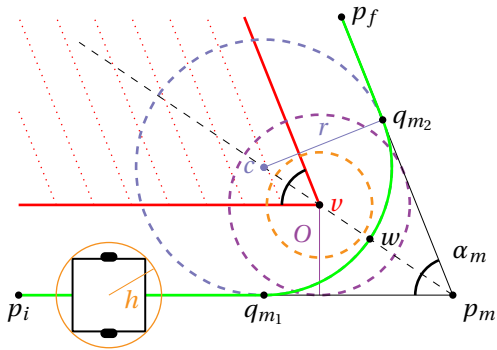

\begin{theorem}[Sufficient condition for a collision-free path]
\label{th:minOffset}
Let $h$ be the radius of the robot, $r$ the minimum turning radius and $\alpha_m$ the angle at the vertex $v$ of a convex obstacle, as shown in Figure~\ref{fig:mitOffsetScenario}. The path $Q$, in green, is collision-free if the obstacles in the environment are subject to a mitered offset~\cite{miter1,miter2,miter3} $O$ of at least:
\begin{equation}
    \label{eq:minDistance}
    O = \max\left(h\sin\frac{\alpha_m}{2} + r\left(1-\sin\frac{\alpha_m}{2}\right),~h\right).
\end{equation}
\end{theorem}

\begin{proof}
The goal is to find an offset $O$ such that the path $Q$, is always collision-free, i.e., the following constraints are satisfied along the whole path:
\begin{equation}
 %   \begin{cases}
       O > h, \qquad
       ||v-w|| > h. 
%    \end{cases}.
    \label{eq:mitConst}
\end{equation}

Since we are considering a mitered offset, the offset $O$ is the radius of a circle with center on the vertex $v$ of the obstacle, as shown in
% Figure~\ref{fig:miterOffset} and 
Figure~\ref{fig:mitOffsetScenario}. 

The second constraint of~\ref{eq:mitConst}, in terms of $r$ and $O$, becomes:
\[
    ||v-w|| = \frac{O}{\sin\left(\alpha_m/2\right)}-(||c-p_m||-r),
\]
where the value $||c-p_m||$ is $d$ from Equation~\ref{eq:lengths}, which can be expressed as a function of $r$:
\[
\begin{array}{rcl}
    ||v-w|| & = & \dfrac{O}{\sin\left(\alpha_m/2\right)} - \dfrac{r}{\sin(\alpha_m/2)} + r \\
            & = & \dfrac{O}{\sin\left(\alpha_m/2\right)} - r\left( \dfrac{1}{\sin(\alpha_m/2)} - 1 \right) > h.
\end{array}
\]
By solving for the offset $O$, we have:
\begin{equation}
    O > h\sin\frac{\alpha}{2} + r\left(1-\sin\frac{\alpha_m}{2}\right).
    \label{eq:offset2}
\end{equation}
%%%
Finally, the offset is the maximum between~\eqref{eq:offset2} and $h$.
\end{proof}

%%%%%%%%%%%%%%%%%%%%%%%%%%%%%%%%%%%%%%%%%%%%%%%%%%%%%%%%%%%%%%%%%%%%%%%%%%%%%%%%%%%%

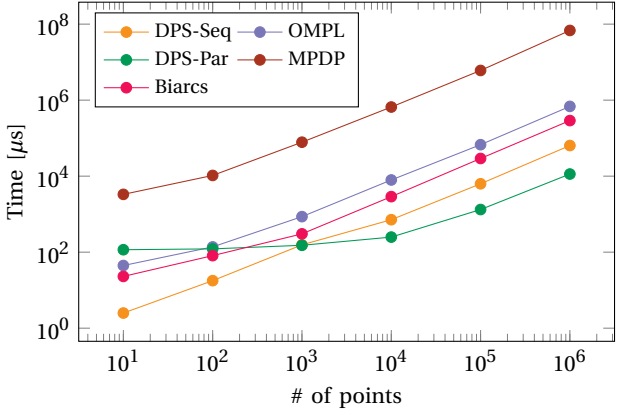
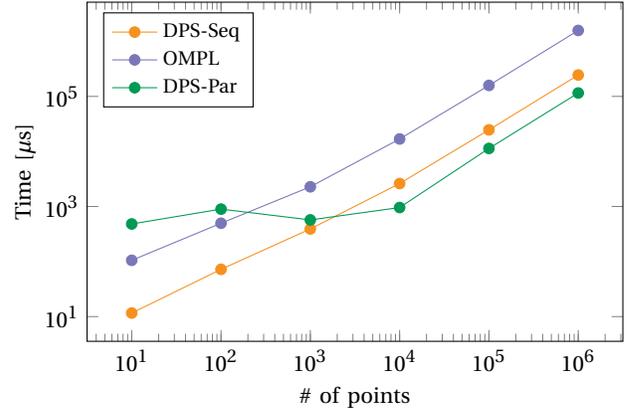
\begin{figure*}[t]
    \centering
    \begin{subfigure}[t]{0.48\linewidth}
        \pgfplotsset{compat=1.18}
\begin{tikzpicture}
\begin{axis}[
    xlabel={\# of points}, 
    ylabel={Time [$\mu$s]},
    xmode=log,
    ymode=log,
    % ylabel near ticks,
    ylabel style={inner sep=0pt},
    legend pos=north west,
    legend style={font=\footnotesize},
    width = \linewidth,
    height = .7\linewidth,
    legend columns = 2,
    legend cell align={left},
    font=\small,
    % ymax=1e10
]
    \addplot [BurntOrange, mark=*] table [y=dubinsSeq, x=nPoints, col sep=comma]{images/data/test1/times.csv};
    \addlegendentry{DPS-Seq}
    \addplot [Periwinkle, mark=*] table [y=dubinsOmpl, x=nPoints, col sep=comma]{images/data/test1/times.csv};
    \addlegendentry{OMPL}
    \addplot [ForestGreen, mark=*] table [y=dubinsCuda, x=nPoints, col sep=comma]{images/data/test1/times.csv};
    \addlegendentry{DPS-Par}
    \addplot [Mahogany, mark=*] table [y=mpdpSaccon, x=nPoints, col sep=comma]{images/data/test1/times.csv};
    \addlegendentry{MPDP}
    \addplot [OrangeRed, mark=*] table [y=biarc, x=nPoints, col sep=comma]{images/data/test1/times.csv};
    \addlegendentry{Biarcs}
\end{axis}
\end{tikzpicture}
        \caption{Tests run on a desktop PC.}
        \label{fig:pathInterpolationTimesPC}
    \end{subfigure}
    \hfill
    \begin{subfigure}[t]{0.48\linewidth}
        \pgfplotsset{compat=1.18}
\begin{tikzpicture}
\begin{axis}[
    xlabel={\# of points}, 
    ylabel={Time [$\mu$s]},
    xmode=log,
    ymode=log,
    % ylabel near ticks,
    ylabel style={inner sep=0pt},
    legend pos=north west,
    legend style={font=\footnotesize},
    width = \linewidth,
    height = .7\linewidth,
    legend cell align={left},
    font=\small
]
    % \addplot table [y=mpdpSaccon, x=nPoints, col sep=comma]{images/data/test1/times_xavier.csv};
    % \addlegendentry{MPDP}
    \addplot [BurntOrange, mark=*] table [y=dubinsSeq, x=nPoints, col sep=comma]{images/data/test1/times_xavier.csv};
    \addlegendentry{DPS-Seq}
    \addplot [Periwinkle, mark=*] table [y=dubinsOmpl, x=nPoints, col sep=comma]{images/data/test1/times_xavier.csv};
    \addlegendentry{OMPL}
    \addplot [ForestGreen, mark=*] table [y=dubinsCuda, x=nPoints, col sep=comma]{images/data/test1/times_xavier.csv};
    \addlegendentry{DPS-Par}
\end{axis}
\end{tikzpicture}
        \caption{Tests run on an embedded system.}            
        \label{fig:pathInterpolationTimesPCXavier}
    \end{subfigure}
    \caption{Performance comparison of the proposed algorithm, OMPL, and MPDP in path interpolation tasks.}
    \label{fig:pathInterpolationTimes}
\end{figure*}

%%%%%%%%%%%%%%%%%%%%%%%%%%%%%%%%%%%%%%%%%%%%%%%%%%%%%%%%%%%%%%%%%%%%%%%%%%%%%%%%%%%%

% \section{Notes on length of the proposed solution}
\begin{lemma}[$\mathbf{Q}$ is shorter than $\mathbf{P}$]
The path $Q$ has a total length that is not longer than the length of the polyline $P$.
\end{lemma}

\begin{proof}
To show that this is true, we must show that the arc of circle is shorter than reaching the point of the polyline and going back:
$
   |\widearc{q_{m_1}q_{m_2}}|= r(\pi-\alpha) \leq 2l.
$
From Equation~\ref{eq:lengths}, it follows that:
\[
    \pi-\alpha \leq \frac{2}{\tan(\alpha/2)},
\]
which is always true for values of $\alpha\in[0,\pi)$. If $\alpha = \pi$, then the points are aligned (excluded by Assumption~\ref{ass:polylineConditions}). 

\end{proof}

\section{Experimental Results}
\label{sec:evaluation}

%%% Figure 6 moved to ProofRP.tex to better distribute the figures

In this section, we present a thorough experimental evaluation of the proposed algorithm. The experiments were conducted on a desktop computer with an Intel i9-10940X processor and an Nvidia RTX A5000 GPU, running Ubuntu 22.04. We divide the tests in different sections in order to highlight the strengths and the weaknesses of our approach compared to other state-of-the-art algorithms. 

\subsection{Path Interpolation}

In this first experiment, we want to show the performance improvement that our algorithm brings over the state of the art ~\cite{8989830,Bertolazzi:2019} both for computation times and for the lengths of the approximated path over interpolation. To do so, we fit $N = 10^i, i\in \{1\hdots 6\}$ points randomly sampled and repeat the tests 1000 times. We show the average time values in Figure~\ref{fig:pathInterpolationTimes} and the average lengths in Table~\ref{tab:pathInterpolationLengths}. 

The points are sampled each time on a circle whose center is the previous point and whose radius is uniformly extracted in $[1,10]$, in such a way that the constraints, shown in Section~\ref{sec:method}, are satisfied.

We show the performance of our approach in comparison to three state-of-the-art algorithms: the Open Motion Planning Library (OMPL)~\cite{OMPL}, the iterative algorithm introduced in~\cite{8989830} that uses dynamic programming to interpolate multiple points with Dubins paths, hereafter referred to as MPDP\footnote{\url{https://github.com/icosac/mpdp}}, and the biarcs implementation\footnote{\url{https://github.com/ebertolazzi/Clothoids}} from~\cite{Bertolazzi:2019}. There are noteworthy differences between these and our algorithms:
\begin{itemize}
    \item OMPL does not support direct interpolation of multiple points, therefore, we exploit DPS to compute the intermediate points $q_i$ and angles $\theta_i$, which OMPL then uses to interpolate the path with Dubins manoeuvres.
    \item MPDP is a sampling-based algorithm designed to interpolate points precisely by finding an approximation of the optimal angles, ensuring that the final path passes through each point. Our algorithm, instead, focuses on finding shorter paths by not passing though the points, but preserving bounded curvature and $G^1$  continuity.
    \item Biarcs are similar to our approach, since they are $G^1$ curves with lines and circles, have computationally efficient closed form solution~\cite{Bertolazzi:2019}, but they do not constrain the curvature.
\end{itemize}
We implement and evaluate two versions of our algorithm: sequential and parallel, which is possible since the paths between triplets of points are independent, as discussed in Section~\ref{sec:method}. The parallel version is implemented in CUDA and is distributed on all the GPU cores available.
Figure~\ref{fig:pathInterpolationTimesPC} shows that our algorithm (both sequentially and in parallel, respectively DPS-Seq and DPS-Par) outperforms in computational times OMPL by approximately one order of magnitude and MPDP by two orders of magnitude. After an initial setback, the parallel implementation (DPS-Par) surpasses OMPL by one order of magnitude and MPDP by two orders of magnitude as the number of interpolated points increases. The initial lack of performance improvement in the parallel implementation is attributed to the time required to transfer data between CPU and GPU memory. From a performance aspect, our implementation appears to be faster in computational time and produces a shorter solution (being optimal), as shown in Table~\ref{tab:pathInterpolationLengths}.
We also run the same tests on a Nvidia AGX Xavier, which is an embedded system frequently used in automotive. Figure~\ref{fig:pathInterpolationTimesPCXavier} shows that DPS maintains the gain in performance shown in Figure~\ref{fig:pathInterpolationTimesPC}.
%%%
\begin{table}[htp]
    \centering
    \begin{tabular}{c|ccccc}
        \# of points & DPS & MPDP$_P$ & MPDP$_Q$ & OMPL & Biarcs \\
        \hline 
        $10^1$          & 1   & 1.040    & 1        & 1    & 1.11 \\
        $10^2$          & 1   & 1.090    & 1        & 1    & 1.22 \\
        $10^3$          & 1   & 1.085    & 1        & 1    & 1.21 \\
        $10^4$          & 1   & 1.086    & 1        & 1    & 1.21 \\
        $10^5$          & 1   & 1.086    & 1        & 1    & 1.21 \\
        $10^6$          & 1   & 1.086    & 1        & 1    & 1.21 \\
    \end{tabular}
    \caption{The ratio between the lengths found using the different algorithms against the lengths found with the sequential algorithm. MPDP$_P$ runs through the points $p$ of the polyline $P$, whereas MPDP$_Q$ runs through the tangent points found using our approach.}
    \label{tab:pathInterpolationLengths}
\end{table}
%%%
\def\baseline{OMPL}
\def\baselineTwo{MPDP}
\def\expOneName{DPS}
%%%
% \begin{figure}[htp]
%     \centering
%     \begin{minipage}{0.47\linewidth}
%         \resizebox{\linewidth}{!}{\input{images/data/test3/map3}}
%     \end{minipage}
%     \hfill
%     \begin{minipage}{0.47\linewidth}
%         \resizebox{\linewidth}{!}{\input{images/data/test3/map4}}
%     \end{minipage}
%     \caption{
%     The maps used for the experiments on the visibility graphs after the mitered-offset. The  objects in black are the obstacles in the maps.
%     \todo[inline]{This may be removed to save space also because at the beginning of Section~\ref{ssec:RoadMaps}, it is written where the maps come from.}
%     }
%     \label{fig:VGMaps}
% \end{figure}
%%%%
\begin{figure*}[t]
    \centering
    % \begin{subfigure}{\linewidth}
        % \centering
        \begin{minipage}{0.48\linewidth}
            \begin{tikzpicture}
    \begin{axis}[
            xlabel={\# of points}, 
            ylabel={Time [$\mu$s]},
            % xmode=log,
            ymode=log,
            % ylabel near ticks,
            ylabel style={inner sep=0pt},
            legend pos=north west,
            legend style={legend columns=-1, font=\footnotesize},
            xtick={2, 4, 6, 8, 10, 12, 14, 16, 18, 20, 22, 24, 26},
            xticklabels={2, 4, 6, 8, 10, 12, 14, 16, 18, 20, 22, 24, 26},
            width=\textwidth,
            height=.6\linewidth,
            ymax=4e7,
            font=\small
        ]
    \addplot [Periwinkle, mark=*] table [y=V, x=N] {dubins_t.tex};
    \addlegendentry{\baseline}
    \addplot [Mahogany, mark=*] table [y=V, x=N] {mpdp_t.tex};
    \addlegendentry{\baselineTwo}
    \addplot [BurntOrange, mark=*] table [y=V, x=N] {pd_t.tex};
    \addlegendentry{\expOneName}
    \end{axis}
\end{tikzpicture}
        \end{minipage}
        \hfill
        \begin{minipage}{0.48\linewidth}
            \begin{tikzpicture}
    \begin{axis}[
            xlabel={\# of points}, 
            ylabel={Time [$\mu$s]},
            % xmode=log,
            ymode=log,
            % ylabel near ticks,
            ylabel style={inner sep=0pt},
            legend pos=north west,
            legend style={legend columns=-1, font=\footnotesize},
            xtick={2, 4, 6, 8, 10, 12, 14, 16, 18, 20},
            xticklabels={2, 4, 6, 8, 10, 12, 14, 16, 18, 20},
            width=\textwidth,
            height=.6\linewidth,
            ymax = 4e7,
            font=\small
        ]
        \addplot [Periwinkle, mark=*] table [y=V, x=N] {dubins4_t.tex};
        \addlegendentry{\baseline}
        \addplot [Mahogany, mark=*] table [y=V, x=N] {mpdp4_t.tex};
        \addlegendentry{\baselineTwo}
        \addplot [BurntOrange, mark=*] table [y=V, x=N] {pd4_t.tex};
        \addlegendentry{\expOneName}
    \end{axis}
\end{tikzpicture}
        \end{minipage}
    % \end{subfigure} \\
    % \begin{subfigure}{\linewidth}
    %     \centering
    %     \begin{minipage}{0.48\linewidth}
    %         \input{images/data/test3/map3_lengths}
    %     \end{minipage}
    %     \hfill
    %     \begin{minipage}{0.48\linewidth}
    %         \input{images/data/test3/map4_lengths}
    %     \end{minipage}
    % \end{subfigure}
    \caption{The results of the experiments run on the visibility graphs showing the times to achieve a solution given the length of the path found on the visibility graph. }
    \label{fig:VGTimes}
\end{figure*}
%%%
In Table~\ref{tab:pathInterpolationLengths}, we experimentally validate that the paths obtained are actually a concatenation of Dubins manoeuvres. Indeed, MPDP yields longer paths interpolating the points $p_j$ (column MPDP$_P$). When used to interpolate the tangent points (column MPDP$_Q$), yields the same path as DPS, thus validating Theorem~\ref{th:DubinsPath}. An identical result is observed by running OMPL through the tangent points and their angles, which also validates the length optimality.

\subsection{Road Maps}
\label{ssec:RoadMaps}

In this experiment, we evaluate the performance of our algorithm DPS in computing a feasible path using a visibility graph on maps with obstacles chosen from~\cite{movingaiBench}.

For this experiment, we first construct the visibility road map~\cite{VisGraphs} by considering the vertices of the inflated obstacles (as described in Theorem~\ref{th:minOffset}), then we use A* to compute the shortest path on the road map and finally we interpolate using our algorithm. We compare our approach to MPDP, to which we must feed the extracted polyline, and to OMPL to which we have to pass the tangent points and their angles, as before.
The experiments were run on 10000 different couples of points, for statistical analysis. In Figure~\ref{fig:VGTimes}, the results show that, even when the number of points is low, our algorithm greatly outperforms the competitors. Moreover, the paths that are extracted are feasible to be followed by construction, as proven in Theorem~\ref{th:minOffset}. 

\subsection{Path Sampling Interpolation}
\renewcommand{\baseline}{OMPL}
\renewcommand{\expOneName}{OMPL-DPS}
\def\expTwoName{RRT*-DPS}

% \begin{figure}[htp]
%     \centering
%     % \includegraphics[width=0.9\linewidth]{images/data/test2/fixed_timeout/3_sec_sum_of_lengths.png}
%     \input{images/data/test2/fixed_timeout/cactus}
%     \caption{A cactus plot showing that our algorithm is able to obtain shorter paths .}
%     \label{fig:RRTOMPLSurvival}
% \end{figure}

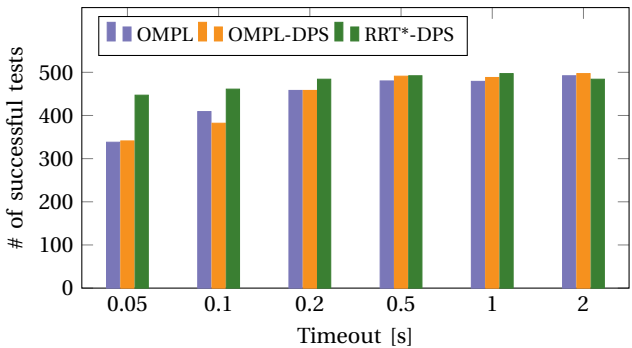
\begin{figure}[htp]
    \centering
    \begin{tikzpicture}
    \begin{axis}[
        ybar=0cm,
        % ybar,
        %
        ymin=0,
        ymax=650,
        ylabel={\# of successful tests},
        xlabel={Timeout [s]},
        xtick=data,
        xticklabels={0.05, 0.1, 0.2, 0.5, 1, 2},
        legend pos=north west,
        legend columns=-1,
        legend style={font=\footnotesize},
        bar width=5pt,
        major grid style={draw=white},
        ymajorgrids, tick align=inside,
        width=\linewidth,
        height=0.6\linewidth,
        ytick={0,100,200,300,400,500},
        yticklabels={0,100,200,300,400,500},
        legend cell align={left},
        font=\small
    ]
    \addplot[Periwinkle, fill, bar shift=-5.4pt] coordinates {(1,338) (2,409) (3,458) (4,480) (5,479) (6,492)};
    \addplot[BurntOrange, fill]   coordinates {(1,341) (2,382) (3,458) (4,491) (5,488) (6,497)};
    \addplot[OliveGreen, fill, bar shift=+5.4pt] coordinates {(1,447) (2,461) (3,484) (4,492) (5,497) (6,484)};
    \legend{\baseline, \expOneName, \expTwoName}
    \end{axis}
\end{tikzpicture}

% 0.05  350         390
% 0.1   450         445
% 0.2   485         465
% 0.5   480         475
% 1.0   475         460
% 2.0   500         480

% Timeout: 0.05 - RRT: 341 - RRT Dubins: 338 - My RRT: 447
% Timeout: 0.1 - RRT: 382 - RRT Dubins: 409 - My RRT: 461
% Timeout: 0.2 - RRT: 458 - RRT Dubins: 458 - My RRT: 484
% Timeout: 0.5 - RRT: 491 - RRT Dubins: 480 - My RRT: 492
% Timeout: 1.0 - RRT: 488 - RRT Dubins: 479 - My RRT: 497
% Timeout: 2.0 - RRT: 497 - RRT Dubins: 492 - My RRT: 484
    \caption{The number of successful cases out of 500 tests at different timeouts.}
    \label{fig:RRTOMPLTimeouts}
\end{figure}

In this experiment, we compare three approaches based on RRT*~\cite{rrtStar}. First, we use the standard RRT* implementation available in OMPL as a baseline, applying it for both sampling the points and interpolating the shortest path with Dubins paths. Next, we present the results for \expOneName, where points are still sampled using RRT* of OMPL, but the shortest path is approximated using our approach. Finally, we introduce \expTwoName, which involves a standard implementation of RRT* from~\cite{rrtStar} coupled with our algorithm. It's important to note that our RRT* implementation does not incorporate heuristic functions or path smoothing (as done by OMPL). However, we employ a 1\% tolerance relative to the map size to reach the goal more precisely. This approach, not used in OMPL, may offer a speed advantage over the other approaches.\\
%
% We carried out two experiments in this setting. First, we compared the lengths obtained by the different approached.  that we set to 3s, an empirically chosen value after evaluating the quality of some solutions. We generated 100 random queries composed of initial and final points, and each query was run 5 times. The results shown in Figure~\ref{fig:RRTOMPLSurvival} are the average values over the repetitions. The results demonstrate that our algorithm manages not only to solve more instances of the problem, but also to find shorter paths.
%
In order for RRT* to terminate, we need to specify a timeout. The experiment consisted in varying the number of timeouts in $\{0.05, 0.1, 0.2, 0.5, 1, 2\}$s, and, for every timeout, we randomly sampled 100 pairs of initial and final points and tested the algorithms. The sampled pairs were run 5 times each and the results shown in Figure~\ref{fig:RRTOMPLTimeouts} are the average values. Our algorithm manages to keep up with the state-of-the-art and outperforms the baseline in multiple timeouts. \\
We note that the bottleneck of path approximation via sampling based methods is the algorithm for sampling. Indeed, when we substitute RRT* from OMPL with a lighter implementation, the performance improvements are notable even for smaller timeouts. While this is indeed expected, the results still show that we are able to reduce the lengths of the paths, while yet maintaining a success rate at least as high as the state-of-the-art one.

\subsection{Physical Robot Experiment}

\begin{figure}[htp]
    \centering
    \includegraphics[width=0.9\linewidth]{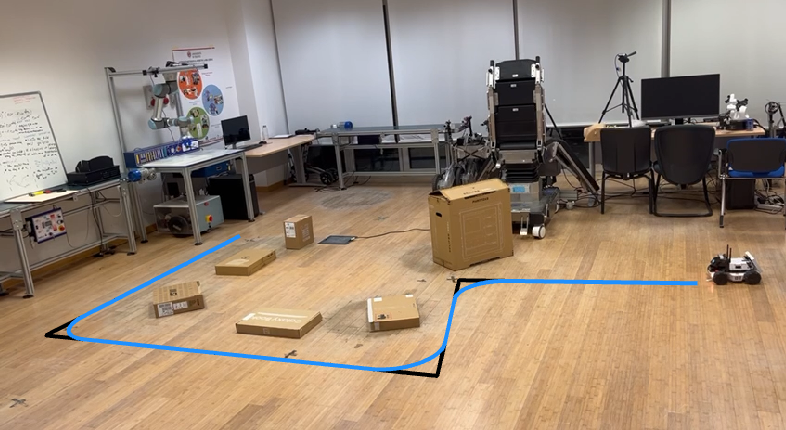}
    \caption{In black, the initial polyline $P$ that the robot was supposed to follow. In blue, the smoothed path $\mathcal{P}$ computed with DPS and followed.}
    \label{fig:real-experiment}
\end{figure}

Finally, we briefly showcase a real-life scenario in which DPS is used by a moving robot to smooth a polyline moving through obstacles.
%
% The video is available at . 
%
For the experiment, we used a car-like vehicle named Agilex Limo, driven with the four-wheel differential mode. We embedded the DPS algorithm implementation in a ROS2 node, which sends command velocity messages to the robot. The lightweight computational requirements to solve the DPS algorithm enables also the Nvidia Jetson Nano, an embedded computer used onboard by the robot, to swiftly find a solution. In Figure~\ref{fig:real-experiment}, one can see the path followed by the car satisfying the above constraints and properties. 

\section{Conclusion}
\label{sec:conclusion}

We have presented DPS, a fast algorithm to generate the shortest smooth approximation of a polylines with $G^1$ continuity and bounded curvature. 
This approach comes with a series of important proved properties, i.e, proofs and conditions on existence, optimality, feasibility, and bounded distance, enabling not only to find an optimal approximation for the path, but also to ensure the quality of such path. 
As a future work, we aim at enabling higher levels of curvature continuity, i.e., $G^2$ with clothoids or $G^3-G^4$ with intrinsic splines.
% \clearpage \newpage
\bibliographystyle{plain}
\nocite{*}
\bibliography{biblio.bib}

\end{document}